\pdfoutput=1
\documentclass{article}

     \PassOptionsToPackage{numbers, compress}{natbib}

\usepackage[preprint]{neurips_2022}
\usepackage{graphicx}
\usepackage{subfigure}
\usepackage{enumitem}
\usepackage{bm}
\usepackage{bbm}
\usepackage{algpseudocode,algorithm,algorithmicx}

\usepackage[utf8]{inputenc} 
\usepackage[T1]{fontenc}    
\usepackage{hyperref}       
\usepackage{url}            
\usepackage{booktabs}       
\usepackage{amsfonts}       
\usepackage{nicefrac}       
\usepackage{microtype}      
\usepackage{xcolor}         
\usepackage{amsmath}
\usepackage{amsthm}

\newtheorem{theorem}{Theorem}

\newtheorem{lemma}[theorem]{Lemma}
\newtheorem{definition}[theorem]{Definition}

\newtheorem{assumption}{Assumption}
\newtheorem{remark}[theorem]{Remark}

\title{Evaluating Aleatoric Uncertainty\\ via
Conditional Generative Models
}

\author{%
  Ziyi Huang, Henry Lam, Haofeng Zhang
    \\
  Columbia University\\
  New York, NY 10027 \\
  \texttt{zh2354, khl2114, hz2553@columbia.edu} \\
}

\begin{document}

\maketitle

\begin{abstract}
Aleatoric uncertainty quantification seeks for distributional knowledge of random responses, which is important for reliability analysis and robustness improvement in machine learning applications. Previous research on aleatoric uncertainty estimation mainly targets closed-formed conditional densities or variances, which requires strong restrictions on the data distribution or dimensionality. To overcome these restrictions, we study conditional generative models for aleatoric uncertainty estimation. We introduce two metrics to measure the discrepancy between two conditional distributions that suit these models. Both metrics can be easily and unbiasedly computed via Monte Carlo simulation of the conditional generative models, thus facilitating their evaluation and training. We demonstrate numerically how our metrics provide correct measurements of conditional distributional discrepancies and can be used to train conditional models competitive against existing benchmarks.

\end{abstract}


\section{Introduction}\label{intro}

Uncertainty quantification plays a pivotal role in machine learning systems, especially for downstream decision-making tasks involving reliability analysis and optimization. There are two major types of uncertainty, \textit{aleatoric} uncertainty and \textit{epistemic} uncertainty. Aleatoric uncertainty refers to the intrinsic stochasticity of the output given a specific input \citep{hullermeier2021aleatoric}, while epistemic uncertainty accounts for the model uncertainty caused by data and modeling limitations \citep{kendall2017uncertainties}. Most classical machine learning algorithms that focus on mean response prediction primarily address epistemic uncertainty, but aleatoric uncertainty, which describes the distribution of responses beyond summary statistics like the mean, has been gaining importance because of risk and safety-critical considerations.

Existing approaches for aleatoric uncertainty estimation can be largely divided into the following directions: negative log-likelihood (NLL) loss-based estimation, forecaster calibration, and conditional density estimation (CDE). While powerful, these approaches are limited by several drawbacks arising from real-world applications:
\begin{enumerate}[leftmargin=*]
\item\emph{Negative Log-Likelihood Loss:} In regression tasks, aleatoric uncertainty can be estimated through the conditional mean and variance from models (heteroscedastic neural networks) optimized by the NLL loss \citep{nix1994estimating,bishop1994mixture,cawley2005estimating,lakshminarayanan2017simple,kendall2017uncertainties,cui2020calibrated}. However, this approach requires scalar-type output, which cannot be easily extended to broader computer vision applications, such as image generation. In addition, the computation of NLL loss relies on assumptions of conditional Gaussian or Gaussian-like distribution, which may not be followed by real-world datasets.
\item\emph{Forecaster Calibration:} In the calibration literature, aleatoric uncertainty estimators are also known as forecasters \citep{gneiting2007probabilistic,kuleshov2018accurate,song2019distribution} with multiple definitions of calibration modes  \citep{gneiting2007probabilistic,song2019distribution,fasiolo2021fast,kuleshov2018accurate,cui2020calibrated}. Under these definitions, the ground-truth conditional distribution function is well calibrated, but not vice versa. Thus, some intuitive sharpness criteria are typically applied to avoid trivial forecasters such as the unconditional distribution. However, little is known about how to recover the ground-truth conditional distribution function via calibration, even asymptotically.
\item\emph{Conditional Density Estimation:} In CDE-based approaches \citep{holmes2012fast,izbicki2017photo,sugiyama2010conditional,pospisil2018rfcde,pospisil2019f,dalmasso2020conditional,dutordoir2018gaussian}, aleatoric uncertainty is directly calculated by estimating conditional densities in a certain form (such as kernel density). Most of CDE methods can only apply on low-dimensional responses following absolutely continuous conditional distributions. Moreover, the output of CDE methods is an explicit formula of the conditional density function. Thus, numerical characteristics such as conditional quantiles may be hard to obtain, as it involves numerical integration that is generally difficult to implement, especially in higher-dimensional settings. 
\end{enumerate}
To address the above challenges, we study a framework using conditional generative models to estimate aleatoric uncertainty. Compared to previous approaches, conditional generative models \citep{mirza2014conditional,ren2016conditional} are more scalable and flexibly applicable regardless of the dimension and distribution of the input/output vector. Moreover, they can easily generate numerical characteristics of the underlying distributions or other performance estimations through Monte Carlo methods. 

At the core of our framework is the construction of distance metrics between the generative model and the ground-truth distribution, which is required for both model evaluation and training \citep{goodfellow2014generative,nowozin2016f,arjovsky2017wasserstein}. In particular, we generalize the maximum mean discrepancy (MMD) \citep{gretton2012kernel,li2015generative} to the setting of conditional distributions, by constructing two new metrics that we call joint maximum mean discrepancy (JMMD) and average maximum mean discrepancy (AMMD). We derive statistical properties in estimating these metrics and illustrate that both metrics admit easy-to-implement and computationally scalable unbiased estimators. Based on these, we further develop two approaches to optimize conditional generative models suited for different tasks and conduct comprehensive experiments to show the correctness and effectiveness of our framework.

Our approach has the following strengths relative to previous methods: 
1) A similar study with conditional MMD can be found in \citep{ren2016conditional} which, as far as we know, is the most relevant work on MMD-based conditional generative models. However, their framework involves unrealistic technical assumptions that may not hold in practice, as well as matrix inversion operations that suffer from instability and scalability issues (see Section \ref{sec:CMMD}). 2) Both JMMD and AMMD are evaluation metrics that are desirably "distribution-free" (i.e., the data are not assumed any particular type of distributions) and "model-free" (i.e., the evaluation does not involve additional estimated models such as the discriminator). In previous research, Fr\'{e}chet Inception Distance (FID) \citep{heusel2017gans,lucic2018gans} is a standard metric to assess the quality of unconditional generative models. However, the closed-form computation of FID assumes that both generative models and data follow multivariate Gaussian distributions. Another commonly used evaluation approach is Indirect Sampling Likelihood (ISL) \citep{breuleux2011quickly,goodfellow2014generative}, which computes the NLL under a fitted kernel density based on generative models. However, kernel density estimation deteriorates in quality when dimensionality increases and could fit poorly into the generative models. Finally, the value of loss on testing data is an alternative for performance examination. However, typical losses such as using $f$-divergence or Wasserstein distance cannot indicate the performance of the generator alone (see Section \ref{related}). 
\section{Related Work}\label{related}

\textbf{Learning and Evaluation Criterion.} Evaluation criteria on generative models against data are typically borrowed from discrepancy measures between two probability distributions in the statistics literature. The latter includes two major types: $f$-divergence and integral probability metrics. The seminal paper \citep{goodfellow2014generative} used Jensen-Shannon divergence in its original form and then \citep{nowozin2016f} extended it to general $f$-divergence motivated from the benefits of other divergence function choices. The computation of integral probability metrics have two important sub-directions, MMD \citep{li2015generative,li2017mmd} and Wasserstein distance \citep{arjovsky2017wasserstein,gulrajani2017improved}. Among these criteria, a discriminator is typically needed for approaches with $f$-divergence (variational representation) and Wasserstein distance (dual representation), while it is not required for MMD methods. The loss function from $f$-divergence and Wasserstein distance cannot be directly use to evaluate generative models alone due to their dependency on the quality of discriminators. In addition, other conditional distance measures may encounter challenges when using generative models. For instance, NLL value \citep{lakshminarayanan2017simple} and CDE value \citep{dalmasso2020conditional} require an explicit form of the model's density function. 


\textbf{Aleatoric Uncertainty in Deep Learning.} Besides the directions discussed in Section \ref{intro}, aleatoric uncertainty on classification tasks can be estimated from the output of softmax layers in deep neural networks \citep{niculescu2005predicting,lakshminarayanan2017simple,hendrycks2016baseline}. Previous research \citep{guo2017calibration} pointed out that directly using softmax outputs for estimation could be inaccurate, as the softmax probability on predicted class did not reflect the ground-truth correctness likelihood. The ground-truth conditional mass function has zero calibration error but not vice versa. Hence forecasters with zero calibration error, which have been studied extensively \citep{kumar2019verified,minderer2021revisiting,kull2019beyond}, are not sufficient to recover the ground-truth conditional mass function. The forecasters could be heuristically improved by a second-level metric named sharpness (or refinement error) \citep{kumar2019verified,kuleshov2015calibrated,kull2015novel}. Since aleatoric uncertainty in classification can be captured by vector-valued maps such as softmax responses, it is not necessary to use a conditional generative model for this task, and thus we do not focus on classification in this paper.

\section{Conditional Generative Models and Maximum Mean Discrepancy}

\subsection{Conditional Generative Models}
In this section, we provide rigorous definitions on conditional generative models. 
Consider a standard statistical framework where a pair of random vectors $(X, Y) \in \mathcal{X} \times \mathcal{Y}$ follows a joint distribution $P_{X,Y}$ with marginal distributions $X \sim P_X$ and $Y \sim P_Y$. We assume the space $\mathcal{X} \subset \mathbb{R}^d$ with $d \ge 1$ which is allowed to contain either continuous or discrete components. Denote the conditional distribution of $Y$ given $X$ by $P_{Y|X}$. For a given value $x$ of $X$, denote the conditional distribution as $P_{Y|X=x}$. Typically, we regard $X$ as a vector of input (example) and $Y$ as a vector of output (label). For instance, $Y \subset \mathbb{R}^q$ with $q \ge 1$ in regression and $Y \subset [K]:=\{1,\ldots,K\}$ in classification. 
Alternatively, in image generation tasks, $X$ refers to auxiliary information (such as the image attributes or labels) and $Y$ refers to the image in order to keep the notation consistent.

Our goal is to quantify the conditional distribution $P_{Y|X}$ via conditional generative models. More precisely, let $\xi \in \mathbb{R}^m$ be a random vector independent of $X$ with a known distribution $P_{\xi}$ (specified by the learner) and the goal is to construct a function $G: \mathbb{R}^m \times \mathcal{X} \to \mathcal{Y}$ such that the conditional distribution of
$G(\xi, X)|X=x$ is the same as $P_{Y|X=x}$. The following lemma demonstrates the existence of such function $G$, termed as the conditional generative model $G: \mathbb{R}^m \times \mathcal{X} \to \mathcal{Y}$.

\begin{lemma} [Adapted from Theorem 5.10 in \citep{kallenberg1997foundations}]
Let $(X, Y)$ be a random pair taking values in
$\mathcal{X} \times \mathcal{Y}$ with joint distribution $P_{X,Y}$. Suppose $Y$ is a standard Borel space. Then there exist a random vector $\xi \sim P_{\xi} = \text{Uniform}([0, 1]^m)$ 
and a Borel-measurable function
$G: \mathbb{R}^m \times \mathcal{X} \to \mathcal{Y}$ such that $\xi$ is independent of $X$ and
$(X, Y) = (X, G(\xi, X))$ almost everywhere. In particular, such $G$ satisfies that $Y|X=x \sim G(\xi, X)|X=x$ for a.e. $x$ with respect to $P_X$.\label{noise}
\end{lemma} 

The conditional generative model can provide more information than standard regression models with single-point prediction. In regression problems, the conditional mean can be estimated by taking the sample mean of multiple draws from $G(\xi_i, X)|X=x$.  
Meanwhile, other numerical characteristics of the underlying target distribution, such as conditional variance and conditional quantile can also be estimated by Monte Carlo sampling from the conditional generative model, beyond what single-point prediction could offer. 


In the rest of this paper, we use $P_{Y|X}$ for the ground-truth conditional distribution and $Q_{Y|X}$ for the distribution of the conditional generative model $G(\xi, X)|X$. We denote $Q_{X,G(\xi, X)}$ as the joint distribution of $(X,G(\xi, X))$. For each given $x$, the generative model is able to generate conditionally independent and identically distributed (i.i.d.) samples $G(\xi_{i}, x)$ from the conditional distribution $Q_{Y|X=x}$. We parametrize the conditional generative model in a hypothesis class $\{G_\theta(\xi, X):\theta\in\Theta\}$ with parameter $\theta$. 
To learn $G(\xi, X)|X$ as an estimate of $P_{Y|X}$, we need a metric to quantify the difference between $G(\xi, X)|X$ and $P_{Y|X}$ using finite training data, which relates to Two-Sample Test. To this end, we will use the (kernel) maximum mean discrepancy (MMD) \citep{gretton2012kernel}, which is described in the next subsection.

\subsection{Two-Sample Test via Maximum Mean Discrepancy}

We review the standard MMD in the setting of unconditional distribution on $\mathcal{Y}$. Section A
provides preliminaries on the reproducing kernel Hilbert space (RKHS). Suppose that $\mathcal{F}_X$ ($\mathcal{F}_Y$) is the RKHS defined on the space $\mathcal{X}$ ($\mathcal{Y}$) with kernel $k_1$ ($k_2$) and feature map $\phi_1$ ($\phi_2$). We adopt the following two basic assumptions throughout this paper (i.e., all theorems make these assumptions without explicit mentioning). Detailed explanations on Assumptions \ref{ass1} and \ref{ass2} can be found in Section A. The Gaussian kernels for instance satisfy both assumptions.

\begin{assumption} \label{ass1}
We assume the following: 1) $k_1(\cdot, \cdot)$ is measurable and $\mathbb{E}_{x\sim P_X}[k_1(x, x)]<\infty$. 2) $k_2(\cdot, \cdot)$ is measurable and $\mathbb{E}_{y \sim P_Y}[k_2(y, y)]<\infty$. Moreover, $\mathbb{E}_{y \sim P_{Y|X=x}}[k_2(y, y)|X=x]<\infty$ for any $x\in \mathcal{X}$.
In addition, these assumptions also hold when replacing the data distribution $P$ by the generative distribution $Q$. 
\end{assumption}

\begin{assumption} \label{ass2}
We assume the following: 1) $k_1$ is characteristic. 2) $k_2$ is characteristic. 3) $k_1\otimes k_2$ is characteristic. 
\end{assumption}

The integral probability metric aims to measure the discrepancy between two distributions. Let $\mathcal{G}$ denote a set of functions $\mathcal{Y} \to \mathbb{R}$. Given two distributions $P_Y$ and $Q_Y$ on $\mathcal{Y}$, the integral probability metric is defined as
$$IPM(P_Y, Q_Y) = \sup_{f\in \mathcal{G}} | \mathbb{E}[f(Y)]-
\mathbb{E}[f(\hat{Y})]|$$
where $Y\sim P_Y$ and $\hat{Y} \sim Q_Y$. MMD is a special case of integral probability metrics, as it chooses $\mathcal{G}$ to be the unit ball in the RKHS $\mathcal{F}_Y$.
Let $\mu_{P_Y}$ denote the kernel mean embedding of $P_Y$ in $\mathcal{F}_Y$: $\mu_{P_Y}:= \mathbb{E}_{y\sim P_Y} [\phi_2(y)]$. $\mu_{P_Y}$ is guaranteed to be an element in the RKHS $\mathcal{F}_Y$ under Assumption \ref{ass1} \citep{gretton2012kernel}. With these discussions, the square of MMD distance between $P_Y$ and $Q_Y$ is formally defined as
\begin{align}
&MMD^2(P_Y, Q_Y) = \sup_{f\in \mathcal{F}_Y, \|f\|_{\mathcal{F}_Y}\le 1} | \mathbb{E}[f(Y_1)]-
\mathbb{E}[f(\hat{Y}_1)]|^2 \nonumber\\
=& \|\mu_{P_Y}-\mu_{Q_Y}\|^2_{\mathcal{F}_Y} = \mathbb{E}[k_2(Y_1, Y_2)]- 2\mathbb{E}[k_2(Y_1, \hat{Y}_1)] + \mathbb{E}[k_2(\hat{Y}_1, \hat{Y}_2)]     \label{equ:MMD}
\end{align}
where $Y_1, Y_2\stackrel{i.i.d.}{\sim} P_Y$ and $\hat{Y}_1, \hat{Y}_2\stackrel{i.i.d.}{\sim} Q_Y$.

\begin{theorem}[\citep{gretton2012kernel}]
$MMD^2(P_Y, Q_Y)\ge 0$ and $MMD^2(P_Y, Q_Y)= 0$ if and only if $P_Y= Q_Y$.
\end{theorem}

Suppose we have data $y_i\stackrel{i.i.d.}{\sim} P_{Y}$ ($i\in [n]$) and $\hat{y}_{j}\stackrel{i.i.d.}{\sim} Q_{Y}$ ($j\in[m]$). Then a standard unbiased estimator of $MMD^2(P_Y, Q_Y)$ \citep{gretton2012kernel} is
{\small $$
\mathcal{L}^u_{MMD^2} = \frac{1}{n(n-1)} \sum_{i=1}^n \sum_{i'\ne i, i'=1}^n k_2(y_i , y_{i'}) - \frac{2}{nm} \sum_{i=1}^n \sum_{j=1}^m k_2(y_i , \hat{y}_j) + \frac{1}{m(m-1)} \sum_{j=1}^m \sum_{j'\ne j, j'=1}^m k_2(\hat{y}_j, \hat{y}_{j'})$$}



\section{Generalization to Conditional Two-Sample Test} \label{sec:MMD}

In this section, we generalize MMD to conditional Two-Sample Test. We first explain the limitation of conditional maximum mean discrepancy (CMMD), the state-of-the-art approach to use MMD for conditional models \citep{ren2016conditional}. Then, we introduce two metrics, JMMD and AMMD, which bypass the limitations of CMMD on strong restrictions and biased estimation. We also present the connections and comparisons among these metrics, and describe how to use them to construct conditional generative models.



\subsection{Previous Work on Conditional Maximum Mean Discrepancy} \label{sec:CMMD}
In \citep{ren2016conditional}, conditional generative moment-matching networks (CGMMNs) were developed for conditional distribution generation. In particular, they leveraged previous work on conditional mean embeddings of the conditional distribution $C_{P_{Y|X}}$ \citep{song2009hilbert,fukumizu2013kernel,song2013kernel,muandet2017kernel} (Section A 
provides a review on this topic.) They used the discrepancy between $C_{P_{Y|X=x}}$ and $C_{Q_{Y|X=x}}$ to measure the difference of two conditional distributions, termed as CMMD, which is defined formally as: 
\begin{equation} \label{equ:CMMD0}
CMMD^2 = \|C_{P_{Y|X}} - C_{Q_{Y|X}}\|^2_{\mathcal{F}_X \otimes \mathcal{F}_Y}    
\end{equation}
where $P$ represents the ground-truth data distribution and $Q$ represents the generative distribution.
The estimator of $CMMD^2$ developed in \citep{ren2016conditional} is as follows:
\begin{equation}  \label{equ:CMMD}
\mathcal{L}_{C^2}(P,Q) = \|\tilde{C}_{P_{Y|X}} - \tilde{C}_{Q_{Y|X}}\|^2_{\mathcal{F}_X \otimes \mathcal{F}_Y},
\end{equation}
\begin{equation} \label{equ:cyx}
\tilde{C}_{P_{Y|X}}= \tilde{C}_{P_{YX}}(\tilde{C}_{P_{XX}}+\lambda I)^{-1}=\Phi_2(K_X + \lambda n I)^{-1} \Phi_1 \footnotemark
\footnotetext{The sample size $n$ should appear in the formula \citep{muandet2017kernel} but seems missing in the paper \citep{ren2016conditional}.}
\end{equation}
where $\Phi_2 = (\phi_2(y_1), ..., \phi_2(y_n))$, $\Phi_1 =
(\phi_1(x_1), ..., \phi_1(x_n))$, $K_X = \Phi_1^\top\Phi_1$, and $I$ is the identity matrix. 

While \citep{ren2016conditional} is the most relevant study to our problem setting, directly applying CMMD on aleatoric uncertainty estimation has following limitations: 

\textbf{Computationally Expensive:} The matrix inverse in the estimator is computationally expensive for practical implementation. The running time for a single inversion in one iteration is at the order of $O(B^3)$, where $B$ is the batch size. Meanwhile, the batch size should be sufficient large to achieve good performance for generative models \citep{li2015generative}. 

\textbf{Strong Technical Assumptions and Existence of Inversion:} 1) The existence of the conditional mean embedding operator $C_{P_{Y|X}}$ typically requires strong assumptions: $\forall g\in\mathcal{F}_Y$, $\mathbb{E}_{P_{Y|X}} [g(Y )|X] \in \mathcal{F}_X$. This assumption is not necessarily true for continuous domains 
\citep{song2009hilbert}, and simple counterexamples using the Gaussian kernel can be found \citep{fukumizu2013kernel}. 
2) In general, $\tilde{C}^{-1}_{XX}$ does not exist when $\mathcal{F}_X$ is infinite dimensional, since $\tilde{C}_{XX}$ is a compact operator and thus has an arbitrary small positive eigenvalue \citep{muandet2017kernel}. When the matrix is singular, the matrix inversion could be unstable and the performance of the estimator $\tilde{C}_{P_{Y|X}}$ in \citep{ren2016conditional} might be degraded after adding $\lambda I$ to avoid the singularity. 
3) Even though the first two points could be relieved in some sense \citep{park2020measure}, the CMMD metric \eqref{equ:CMMD0} is well-defined only if $C_{P_{Y|X}}, C_{Q_{Y|X}}\in \mathcal{F}_X \otimes \mathcal{F}_Y$. However, this requires a much stronger assumption than the existence of $\tilde{C}^{-1}_{XX}$ (See Assumption 6 in Section A).

\textbf{Bias:} CMMD does not admit any obvious unbiased estimator. The estimator $\tilde{C}_{P_{Y|X}}$ in \eqref{equ:cyx} is biased, even in the asymptotic sense if $\lambda$ is fixed \citep{song2009hilbert}.


To bypass the above limitations, we propose two alternative metrics which only require basic assumptions on the existence of the cross-covariance operator and the characteristic property of the kernels (Assumptions \ref{ass1} and \ref{ass2}). In particular, we do not require the existence of the inversion of any operator or matrix, which makes our metrics easily implemented for real-world applications.

\subsection{Average Maximum Mean Discrepancy (AMMD)}
We first introduce a rather straightforward approach, which we term the AMMD metric. AMMD shows better potential for multi-output problems (such as image generation) where data consists of i.i.d. inputs $x_i$ with conditionally independent outputs $y_{i,j}$ at each $x_i$; see Section C for a more detailed discussion. In AMMD, at each $x$, we use \eqref{equ:MMD} to build unbiased estimators of the MMD on the conditional distribution of $Y|X=x$. Then, these estimators are averaged with respect to the marginal $P_X$. 
More specifically, we define
$$AMMD^2(P, Q)=\mathbb{E}_{x \sim P_X}[MMD_{X=x}^2(P_{Y|X=x}, Q_{Y|X=x})]$$
where
\begin{eqnarray}
&&MMD_{X=x}^2(P_{Y|X=x}, Q_{Y|X=x}) := \|\mu_{P_{Y|X=x}}-\mu_{Q_{Y|X=x}}\|^2_{\mathcal{F}_Y}\nonumber\\
&= &\mathbb{E}[k_2(Y^{x}_1, Y^{x}_2)|X=x]- 2\mathbb{E}[k_2(Y^{x}_1, \hat{Y}^{x}_1)|X=x] + \mathbb{E}[k_2(\hat{Y}^{x}_1, \hat{Y}^{x}_2)|X=x] \label{equ:MMDX}     
\end{eqnarray}
is a function of $x$ and for fixed $x$, $Y^{x}_1, Y^{x}_2\stackrel{i.i.d.}{\sim} P_{Y|X=x}$, and $\hat{Y}^{x}_1, \hat{Y}^{x}_2\stackrel{i.i.d.}\sim Q_{Y|X=x}$. Note that $\mu_{P_{Y|X=x}}$, $\mu_{Q_{Y|X=x}}\in \mathcal{F}_Y$ guaranteed by Assumption \ref{ass1} so $MMD_{X=x}^2$ is well-defined. Hence $AMMD^2(P, Q)$ is also well-defined being the expectation with respect to non-negative measurable functions. 

\begin{remark}
In \eqref{equ:MMDX}, $Y^{x}_1, Y^{x}_2$ are drawn in a \textbf{conditionally independent} manner for each $x$. This is not equivalent to globally draw two unconditionally independent samples $Y_1, Y_2$ and consider $Y_1|X=x, Y_2|X=x$ for each $x$ because the latter is not conditionally independent in general. Therefore, we have that in general,
$\mathbb{E}_{x \sim P_X}[\mathbb{E}[k_2(Y^{x}_1, Y^{x}_2)|X=x]]\ne \mathbb{E}[k_2(Y_1, Y_2)].$
\end{remark}

\begin{theorem} \label{thm:AMMD}
$AMMD^2(P, Q)\ge 0$ and $AMMD^2(P, Q)= 0$ if and only if for a.e. $x$ with respect to $P_X$, $P_{Y|X=x}= Q_{Y|X=x}$.
\end{theorem}

Theorem \ref{thm:AMMD} shows that $AMMD^2(P, Q)$ offers a metric to measure $P_{Y|X}$ versus $Q_{Y|X}$. Next, we propose a Monte Carlo estimator of $AMMD^2$ for conditional generative models: 
\begin{enumerate}[leftmargin=*]
\item Take a batch $\{(x_i,y_{i,l}):i\in [n], l\in [r]\}$ from $P$ of batch size $rn$ where $y_{i,l}$ ($l\in [r]$) are the outputs at the same $x_i$. Here, $r$ is restricted by the specification of the task: $r=1$ in single-output problems but can be greater than $1$ in multi-output problems such as image generation; $n\ge 1$.
\item Generate a batch $\{(x_i,G(\xi_{i,j},x_{i})):i\in [n], j\in [m]\}$ from $Q$ of batch size $mn$ where $\xi_{1,1},\ldots, \xi_{n,m}$ are i.i.d. and independent of $x_1,\ldots,x_n$; $m\ge 2$. 
\item Compute 
{ \begin{align}
\hat{A^2}(P, Q)=&\frac{1}{n}\sum_{i=1}^{n} \Big( - \frac{2}{mr} \sum_{j=1}^m \sum_{l=1}^r k_2(y_{i,l}, G(\xi_{i,j}, x_i))\nonumber\\
&+ \frac{1}{m(m-1)} \sum_{j=1}^m \sum_{j'\ne j, j'=1}^m k_2(G(\xi_{i,j},x_i),G(\xi_{i,j'},x_i)) \Big)\label{equ:lossA}
\end{align}}
\end{enumerate}


The next theorem establishes the error analysis of the estimator $\hat{A^2}(P, Q)$. 
\begin{theorem} \label{thm:AMMD2}
$\hat{A^2}(P, Q)$ is an unbiased estimator of $AMMD^2(P, Q)-C_0$ where $C_0$ is a constant independent of $Q$ given by $C_0=\mathbb{E}_{x \sim P_X}[\mathbb{E}[k_2(Y^{x}_1, Y^{x}_2)|X=x]]$. Moreover, the variance of $\hat{A^2}(P, Q)$ is $O(\frac{1}{n\min\{m,r\}})+\frac{1}{n}K_0$ where $K_0=\text{Var}_{x\sim P_X}[- 2\mathbb{E}[k_2(Y^{x}_1, \hat{Y}^{x}_1)|X=x] + \mathbb{E}[k_2(\hat{Y}^{x}_1, \hat{Y}^{x}_2)|X=x]]$ is independent of $n,m,r$. 
(The explicit formula of the variance is given in Section B.)
\end{theorem}
$C_0$ in Theorem \ref{thm:AMMD2} is free of the conditional generative model and thus does not need to be embodied at the training/evaluation phase. This is in the same spirit of NLL which is formed by a free-of-model constant and the Kullback–Leibler divergence between the model and data. 
Theorem \ref{thm:AMMD2} shows that if $n$ is not allowed to be large (e.g., $n$ is bounded above by the number of class in the label-based image generation problems), the variance of the estimator $\hat{A^2}(P, Q)$ should be reduced by increasing $m$ and $r$. On the other hand, if $n$ is allowed to be large (e.g., in regression problems with continuous features), then given a fixed computational budget, we should increase $n$ while maintaining the small possible values $m=2$ and $r=1$ to reduce the variance of $\hat{A^2}(P, Q)$ efficiently.  

\subsection{Joint Maximum Mean Discrepancy (JMMD)} 
We then introduce the JMMD metric, which is based on the joint distribution. Compared with AMMD, JMMD is more suitable for single-output tasks (such as regression) where data consists of joint i.i.d. samples $(x_i,y_i)$; see Section C for a more detailed discussion. According to the observation in Lemma \ref{noise}, the matching of $Q_{Y|X=x}$ (the conditional distribution of
$G(\xi, X)|X=x$) with $P_{Y|X=x}$ for a.e. $x$ can be transferred to the matching of $Q_{X,Y}$ (the joint distribution of $(X, G(\xi, X))$) with $P_{X,Y}$. 
This motivates us to define the following metric which we term as JMMD:
\begin{eqnarray*}
&&JMMD^2(P, Q)=MMD^2(P_{X,Y}, Q_{X,Y})\\
&=& \mathbb{E}[k_3((X_1,Y_1), (X_2,Y_2))]- 2\mathbb{E}[k_3((X_1,Y_1), (\hat{X}_1,\hat{Y}_1))] + \mathbb{E}[k_3((\hat{X}_1,\hat{Y}_1), (\hat{X}_2,\hat{Y}_2))]
\end{eqnarray*}
where $k_3=k_1\otimes k_2$ is the kernel of the tensor product space $\mathcal{F}_X \otimes \mathcal{F}_Y$ and $(X_1,Y_1), (X_2,Y_2)\stackrel{i.i.d.}{\sim} P_{X,Y}$ and $(\hat{X}_1,\hat{Y}_1), (\hat{X}_2,\hat{Y}_2)\stackrel{i.i.d.}{\sim} Q_{X,Y}$. Note that $JMMD^2(P, Q)$ can be viewed alternatively as the discrepancy of the cross-covariance operators $C_{P_{YX}}$, $C_{Q_{YX}}$ defined on the tensor product space $\mathcal{F}_X \otimes \mathcal{F}_Y$. Since $C_{P_{YX}}$, $C_{Q_{YX}}\in \mathcal{F}_X \otimes \mathcal{F}_Y$ guaranteed by Assumption \ref{ass1}, $JMMD^2(P, Q)$ is well-defined (see Section A for more details).

\begin{theorem} \label{thm:JMMD}
$JMMD^2(P, Q)\ge 0$ and $JMMD^2(P, Q)= 0$ if and only if for a.e. $x$ with respect to $P_X$, $P_{Y|X=x}= Q_{Y|X=x}$.
\end{theorem}

Theorem \ref{thm:JMMD} shows that $JMMD^2(P, Q)$ offers a metric to measure $P_{Y|X}$ versus $Q_{Y|X}$. In parallel to AMMD, we propose a Monte Carlo estimator of $JMMD^2$ for conditional generative models: 
Take a batch of samples from $P$ of batch size $r$ $\{(x_l,y_l):l\in [r]\}$; $r\ge 2$. Generate a batch from $Q$ of batch size $m$ $\{(\hat{x}_j,G(\xi_j,\hat{x}_j)):j\in [m]\}$ where $\xi_1,..., \xi_m$ are i.i.d. and independent of $x_1,\ldots,x_r,\hat{x}_1,\ldots,\hat{x}_m$; $m\ge 2$. Compute
{ \begin{align}
&\hat{J^2}(P, Q)=
- \frac{2}{mr} \sum_{j=1}^m \sum_{l=1}^r  k_3((x_l,y_l) , (\hat{x}_j,G(\xi_j,\hat{x}_j))) \nonumber\\
& + \frac{1}{m(m-1)} \sum_{j=1}^m \sum_{j'\ne j, j'=1}^m k_3((\hat{x}_{j},G(\xi_{j},\hat{x}_{j})), (\hat{x}_{j'},G(\xi_{j'},\hat{x}_{j'}))) \label{equ:lossJ}
\end{align}}
The next theorem establishes the error analysis of the estimator $\hat{J^2}(P, Q)$. 
\begin{theorem} \label{thm:JMMD2}
$\hat{J^2}(P, Q)$ is an unbiased estimator of $JMMD^2(P, Q)-C_1$ where $C_1$ is a constant independent of $Q$ given by $C_1=\mathbb{E}[k_3((X_1,Y_1) , (X_2,Y_2))]$. Moreover, the variance of $\hat{J^2}(P, Q)$ is $O(\frac{1}{\min\{m,r\}})$. (The explicit formula of the variance is given in Section B.)
\end{theorem}

$C_1$ in Theorem \ref{thm:JMMD2} is free of the conditional generative model. Theorem \ref{thm:JMMD2} shows that the variance of $\hat{J^2}(P, Q)$ is decreasing at the order of $\frac{1}{\min\{m,r\}}$. Therefore, given a fixed computational budget $B$, we should set $m=\Theta(r)=\Theta(\sqrt{B})$ to achieve the minimum variance of the estimator $\hat{J^2}(P, Q)$. 


\begin{algorithm}[t]
  \caption{Algorithm Framework of A-CGM}  \label{algo}
 \hspace*{\algorithmicindent} \textbf{Input:} Training Dataset $\mathcal{D} = \{(x_i, y_i): i\in \mathcal{I}\}$. \\
 \hspace*{\algorithmicindent} \textbf{Output:} Finalized parameters $\theta$ in the generative model $G_\theta(\xi, x)$.  
  \begin{algorithmic}[1]
\State Randomly divide the training dataset $\mathcal{D}$ into mini batches.
\For{$t = 0, \ldots, T-1$} 
\State Set $\mathcal{B}^G=\emptyset$
\For{each mini batch $\mathcal{B}$ in $\mathcal{D}$} 
\For{each $x \in \mathcal{B}$}
\State Draw multiple i.i.d. copies $\xi_1,\ldots,\xi_m$ from $P_{\xi}$ 
\State Generate conditional samples by forward-propagating through $G_\theta(\xi_j, x)$
\State Add $(x, G_\theta(\xi_1, x)), \ldots, (x, G_\theta(\xi_m, x))$ into $\mathcal{B}^G$
\EndFor
\State Optimize $\theta$ by $\hat{A^2}(P, Q_\theta)$ in \eqref{equ:lossA} based on $\mathcal{B}$ and $\mathcal{B}^G$
\EndFor
\EndFor

  \end{algorithmic}  
\end{algorithm}

\vspace{-0.3em}
\subsection{Connections and Comparisons among Metrics}
\label{sec:connection}
\vspace{-0.3em}
We establish the theoretical connections among CMMD, JMMD, and AMMD as follows.


\begin{theorem} \label{thm:connection1}
Suppose that $AMMD$ and $JMMD$ are well-defined.
Then we have that
$$JMMD^2\le \mathbb{E}_{x\sim P_X}[k_1(x,x)] AMMD^2.$$
\end{theorem}

\begin{theorem} \label{thm:connection2}
Suppose that $AMMD$ is well-defined. Moreover, suppose that for all $g\in\mathcal{F}_Y$, $\mathbb{E}_{P_{Y|X}} [g(Y )|X] \in \mathcal{F}_X$ and $\mathbb{E}_{Q_{Y|X}} [g(Y )|X] \in \mathcal{F}_X$ so that the conditional mean embeddings $C_{P_{Y|X}}$, $C_{Q_{Y|X}}$ are well-defined. Furthermore, we assume $C_{P_{Y|X}}\in \mathcal{F}_X \otimes \mathcal{F}_Y$, $C_{Q_{Y|X}}\in \mathcal{F}_X \otimes \mathcal{F}_Y$ so that CMMD \eqref{equ:CMMD0} is well-defined. Then we have
$$AMMD^2
\le \mathbb{E}_{x\sim P_X} [k_1(x,x)] CMMD^2.$$
\end{theorem}

We further highlight the strengths of AMMD and JMMD on conditional generative model evaluation which is a challenging task due to delicacies of evaluation at the conditional distribution level. First, both metrics are "distribution-free", i.e., the data or conditional generative model are not restricted to a specific type of distributions. In contrast, FID for instance requires the Gaussian assumption. Second, they are "model-free", i.e., their evaluation does not involve additional estimated models beyond the conditional generative model itself, such as kernel density estimators in Indirect Sampling Likelihood \citep{breuleux2011quickly,goodfellow2014generative} or estimated discriminators \citep{nowozin2016f,arjovsky2017wasserstein}.




\vspace{-0.3em}
\subsection{Conditional Generative Model Construction} \label{sec:CGM}
\vspace{-0.3em}
With the evaluation metrics, we present two corresponding deep-learning-based methods to construct conditional generative models. Our approaches are named as J-CGM and A-CGM, with special targets on the JMMD/AMMD for different tasks. 
For performance measurement, the values of JMMD and AMMD are estimated by drawing samples from the generative model optimized in J-CGM/A-CGM. Denote $G_\theta(\xi, X)$ as the generative model optimized in J-CGM/A-CGM with parameters $\theta$. Note that $G_\theta(\xi, X)$ takes both the given conditional variables $X$ and the extra random vector $\xi$ as inputs. Let $Q_\theta$ be the joint distribution of $(X, G_\theta(\xi, X))$. A detailed step-by-step pseudo-code of A-CGM is listed in Algorithm \ref{algo}. A similar procedure for J-CGM is presented in Algorithm 2 in Section C.

\vspace{-0.3em}
\section{Experiments} \label{sec:exp}
\vspace{-0.3em}

\textbf{Experimental Setup.} We empirically verify the effectiveness of our proposed approaches on both regression and image generation tasks.  In both tasks, we compare the performance of our approaches with the state-of-the-art MMD-based conditional generative model, CGMMN \citep{ren2016conditional}. For regression, our experiments are conducted on the following widely used real-world benchmark  datasets: Boston, Concrete, Energy, Wine, Yacht, Kin8nm, Protein, and CCPP \citep{hernandez2015probabilistic,gal2016dropout,lakshminarayanan2017simple,pearce2018high}. Besides the JMMD and AMMD values, we also report the scores of FID \citep{heusel2017gans,lucic2018gans}, which is a standard metric for assessing the quality of generative models. In our experiments, FID is computed based on the joint distribution of $(X,Y)$, since it is originally defined in the unconditional sense. 
In the label-based image generation task, we adopt the benchmark dataset MNIST \citep{lecun1998gradient} to evaluate the correctness of our framework. In this task, $X$ is the label of the image and the generative model $G_\theta(\xi, X)$ is expected to output random image samples with the attribute of class $X$. 
We provide visuals to directly show the generation performance of different approaches. All experiments are conducted on a GeForce RTX 2080 Ti GPU. More experimental results are presented in Section C.
\vspace{-0.3em}
\subsection{Aleatoric Uncertainty in Regression}
\vspace{-0.3em}
\begingroup
\tabcolsep = 0.5pt

\begin{table}[]
    \centering
{\scriptsize   
    \begin{tabular}{c|ccc|ccc|ccc}
       & \multicolumn{3}{|c|}{CGMMN} & \multicolumn{3}{|c|}{J-CGM} & \multicolumn{3}{|c}{A-CGM}\\
 Dataset & JMMD  & AMMD & FID & JMMD  & AMMD & FID & JMMD  & AMMD & FID\\
 \midrule
     Boston  & $1.92 \cdot 10^{-2} $  & \bm{$8.41 \cdot 10^{-5}$} & $1.35 \cdot 10^{-2} $& \bm{$2.92 \cdot 10^{-4}$}  & $8.49 \cdot 10^{-4}$ &  $7.84 \cdot 10^{-3}$ & $3.17 \cdot 10^{-4}$ & $8.74 \cdot 10^{-5}$ & \bm{$5.20 \cdot 10^{-3}$} \\
     
     Concrete  & $9.57 \cdot 10^{-3}$  & $1.67 \cdot 10^{-4}$ & \bm{$8.64 \cdot 10^{-3}$} & \bm{$1.53 \cdot 10^{-4}$}  & $1.78 \cdot 10^{-4}$ & $9.86 \cdot 10^{-3}$ & $2.47 \cdot 10^{-4}$  & \bm{$1.99 \cdot 10^{-5}$} & $9.74 \cdot 10^{-3}$\\
     
     Energy & $1.87 \cdot 10^{-2}$  & $1.40 \cdot 10^{-4}$ & \bm{$7.96 \cdot 10^{-3}$} & $3.16 \cdot 10^{-4}$  & $9.45 \cdot 10^{-4}$ & $9.16 \cdot 10^{-3}$ & \bm{$3.09 \cdot 10^{-4}$}  & \bm{$1.23 \cdot 10^{-4}$} & $9.38 \cdot 10^{-3}$\\
     
     Wine & $1.09 \cdot 10^{-2}$  & $3.02 \cdot 10^{-4}$ & $1.01 \cdot 10^{-2}$ & \bm{$1.14 \cdot 10^{-4}$}  & $3.61 \cdot 10^{-4}$ & $9.94 \cdot 10^{-3}$ & $1.16 \cdot 10^{-4}$  & \bm{$2.72 \cdot 10^{-4}$} & \bm{$9.86 \cdot 10^{-3}$}\\
     
     Yacht & $1.28 \cdot 10^{-2}$  & $5.76 \cdot 10^{-5}$ & $1.11\cdot 10^{-2}$ & $6.60\cdot 10^{-4}$ & $4.75\cdot 10^{-4}$  &  $1.13\cdot 10^{-2}$ & \bm{$1.67 \cdot 10^{-4}$}  & \bm{$4.63 \cdot 10^{-5}$} & \bm{$7.68 \cdot 10^{-3}$}\\
     
     Kin8nm & $1.00\cdot 10^{-2}$  & $1.46\cdot 10^{-3}$ & $1.41 \cdot 10^{-2}$ & $1.10 \cdot 10^{-4}$ & $1.39 \cdot 10^{-3}$ & \bm{$9.69 \cdot 10^{-3}$} & \bm{$1.01 \cdot 10^{-4}$}  & \bm{$1.38 \cdot 10^{-3}$} & $9.76 \cdot 10^{-3}$\\
     
     Protein & $ 9.20 \cdot 10^{-3}$  & $7.87 \cdot 10^{-3}$& $9.83 \cdot 10^{-3}$ & \bm{$7.08 \cdot 10^{-5}$} & $8.08 \cdot 10^{-3}$& \bm{$9.81 \cdot 10^{-3}$} & $8.60 \cdot 10^{-5}$  & \bm{$2.18 \cdot 10^{-3}$} & $1.00 \cdot 10^{-2}$\\
     
     CCPP & $ 1.64 \cdot 10^{-2}$  & $2.42 \cdot 10^{-3}$& $ 9.99\cdot 10^{-3}$ & $ 2.91\cdot 10^{-4}$ & $2.24 \cdot 10^{-3}$& $ \bm{9.52\cdot 10^{-3}$} & \bm{$2.57 \cdot 10^{-4}$}  & \bm{$ 1.59\cdot 10^{-3}$} & $1.02\cdot 10^{-2}$\\
    \end{tabular}
    \caption{Conditional generative models in regression. Best results are in \textbf{bold}.}
    \label{tab:my_table}}
\end{table}

\textbf{Kernel Selections.} As the regression data are low-dimensional, we choose $k_1$ and $k_2$ to be the standard Gaussian kernels $k_1(x_1,x_2):= \exp\left(-\frac{1}{2}\|x_1-x_2\|_2^2\right)$, $k_2(y_1,y_2):= \exp\left(-\frac{1}{2}\|y_1-y_2\|_2^2\right)$. Note that they readily satisfy both Assumptions \ref{ass1} and \ref{ass2}.

\textbf{Implementation Details.} For regression tasks, we apply a simple network architecture with 2 hidden layers to avoid overfitting. We use the ReLU function
as the activation function and the number of neurons in each hidden layer is 32. The input of the generative model is concatenated by two vectors, the covariate vector $X$ and the extra random vector $\xi$ following a $10$-dimensional uniform distribution $\text{Uniform}([-1,1]^{10})$. Our network is optimized by the Adam optimizer \citep{kingma2014adam} with learning rate 0.0005. For the preprocessing step, we follow the same experimental procedure in \citep{pearce2018high} on data normalization and dataset splitting. The AMMD evaluation omits the free-of-model constant $C_0$ as justified in Theorem \ref{thm:AMMD2}.

We evaluate the performance of our proposed J-CGM and A-CGM with the state-of-the-art baseline CGMMN \citep{ren2016conditional} on multiple real-world benchmark regression datasets. Table \ref{tab:my_table} reports the evaluation metrics from different models on the testing data. As shown, J-CGM and A-CGM achieve competitive performance under the AMMD and JMMD evaluation criteria, while A-CGM tends to produce better results on AMMD. In contrast, although CGMMN can produce satisfactory results on AMMD, it underperforms on JMMD in general. Under the FID criterion, J-CGM and A-CGM achieve slightly better results than CGMMN. However, note that FID is a heuristic criterion without the statistical properties we develop for AMMD and JMMD and is thus less reliable.


\vspace{-0.3em}
\subsection{Aleatoric Uncertainty in Image Generation}
\vspace{-0.3em}

\begin{figure}[t]
    \centering
    \includegraphics[width=1.0\textwidth]{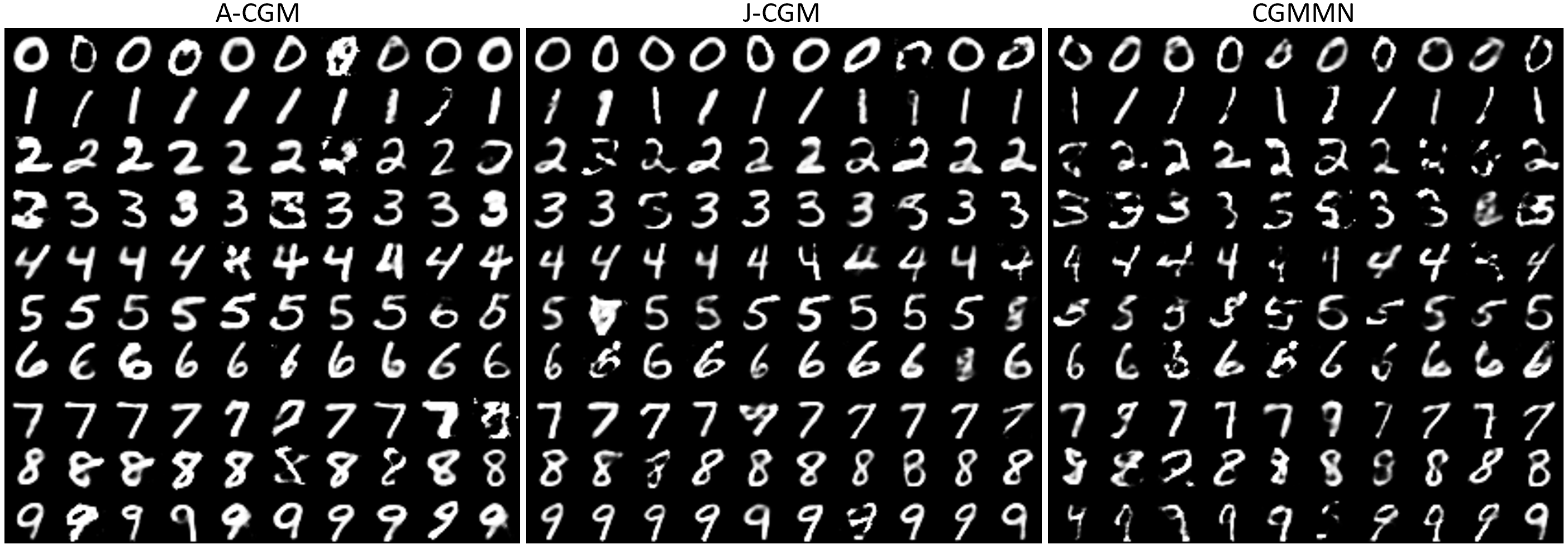}
    \caption{Random conditional samples generated by different approaches.}
    \label{myfigure}
\end{figure}


\textbf{Kernel Selections.} As pointed out by previous studies on deep kernels \citep{li2015generative,li2017mmd,wilson2016deep,liu2020learning,gao2021maximum}, for complicated and high-dimensional real-world data, a kernel test using a simple kernel such as Gaussian kernel should be conducted on the code/feature space instead of the original data space to provide stronger signals for discrepancy measurement of high-dimensional distributions. Following this guidance, we apply an auto-encoder \citep{rumelhart1985learning} to learn the representative features of the input images in the preprocessing step. Precisely, suppose the pre-trained auto-encoder network is given by $B_{\omega'}\circ A_\omega$ where
$A_\omega: \mathcal{Y}\to \hat{\mathcal{Y}}$ is the encoder network with $\hat{\mathcal{Y}}$ being the lower-dimensional code space; $B_{\omega'}: \hat{\mathcal{Y}}\to \mathcal{Y}$ is the decoder network. We use the following feature-aware deep kernel for the image space $\mathcal{Y}$:
$k_2(y_1, y_2)=  \Big((1 - \epsilon_0) \kappa_1(A_\omega(y_1), A_\omega(y_2)) + \epsilon_0 \Big)\kappa_2(y_1, y_2),$
where $\kappa_1$ is a Gaussian kernel defined on the code space $\hat{\mathcal{Y}}$; $\kappa_2$ is a Gaussian kernel defined on the original image space $\mathcal{Y}$;  $\epsilon_0 \in (0, 1)$ is introduced to ensure that $k_2(y_1, y_2)$ is a characteristic kernel \citep{liu2020learning,gao2021maximum}. We set $k_1$ to be
the standard Gaussian kernels since $\mathcal{X}$ is low-dimensional.

Corresponding to our kernel, we now assume that all conditional generative models output samples in the code space for the convenience of MMD tests: $G_\theta(\xi, X): \mathbb{R}^m \times \mathcal{X}\to \hat{\mathcal{Y}}$.
The generative image is then given by $B_{\omega'} \circ G_\theta(\xi, X)$.

\textbf{Implementation Details.} In the auto-encoder network, the encoder/decoder networks $A_\omega$ and $B_{\omega'}$ have a single hidden layer with 1024 neurons. The dimension of the code space $\hat{\mathcal{Y}}$ is 32. The generative network is formed by 3 hidden layers with ReLU function
as the activation function. The number of neurons in each hidden layer is 64, 256, and 256. The networks also take two vectors as input, the one-hot encoding vector of label $X$ and the extra random vector $\xi$ following a $10$-dimensional uniform distribution $\text{Uniform}([-1,1]^{10})$. The generative network is optimized by the Adam optimizer \citep{kingma2014adam} with learning rate 0.001. 

In Figure \ref{myfigure}, we show a few random conditional samples of the reconstructed images from A-CGM, J-CGM, and CGMMN. Overall, all models can generate clear and recognizable samples of handwritten digits. In particular, the reconstructed images from J-CGM are more diverse with multiple writing types, while those from A-CGM are more clearly distinct. These results evidently demonstrate the effectiveness of our approaches on multiple real-world applications. 

\vspace{-0.3em}
\section{Conclusions} \label{sec:con}
\vspace{-0.3em}
In this paper, we study the feasibility of leveraging conditional generative model on aleatoric uncertainty estimation. With theoretical justification, we propose two metrics for discrepancy measurement between two conditional distributions and demonstrate that both metrics can be easily and unbiasedly computed via Monte Carlo simulation. Experimental evaluations on multiple tasks corroborate our theory and further demonstrate the effectiveness of our approaches on real-world applications.  Our study explores a new direction on aleatoric uncertainty estimation, which overcomes a few limitations in the previous research. In the future, we will extend our approaches for aleatoric uncertainty estimation on more real-world applications such as super-resolution image generation.



\medskip

\bibliographystyle{abbrv}
\bibliography{bib}

\end{document}